\documentclass[letterpaper, 10 pt, conference]{ieeeconf} 

\IEEEoverridecommandlockouts 

\usepackage{graphics} 
\usepackage{epsfig} 
\usepackage{amsmath} 
\usepackage{amssymb} 
\usepackage{amsfonts}
\usepackage{mathtools}
\usepackage{algorithmic}
\usepackage{algorithm}
\usepackage{array}
\usepackage{textcomp}
\usepackage{stfloats}
\usepackage{url}
\usepackage{verbatim}
\usepackage{graphicx}
\usepackage{cite}
\usepackage[inkscapeformat=png]{svg}
\usepackage{subcaption}
\usepackage[font={scriptsize}]{caption}
\usepackage[section]{placeins}
\usepackage{float}
\usepackage{booktabs}
\usepackage{makecell}
\usepackage{blindtext}
\usepackage{comment}
\usepackage{tablefootnote}
\usepackage{footnotehyper}
\usepackage[colorlinks,linkcolor=blue]{hyperref}
\newtheorem{theorem}{Theorem}[section]

\newtheorem{lemma}[theorem]{Lemma}
\graphicspath{{./figures/}}

\usepackage{balance}
\usepackage{multirow}
\usepackage{framed}
\usepackage{wrapfig}
\usepackage{import}
\usepackage{arydshln}
\usepackage{multicol}
\usepackage{hyperref}
\hypersetup{
 colorlinks=true,
 linkcolor=black,
 citecolor=black,
 filecolor=magenta, 
 urlcolor=blue,
 }

\newcommand{\reals}{\mathbb{R}}

\title{\Large\bf
Adaptive Robot Detumbling of a Non-Rigid Satellite}
\author{Longsen Gao$^{1}$, Claus Danielson$^{2}$, and Rafael Fierro$^{1}$
\thanks{This material is based on research sponsored by Air Force Research Laboratory (AFRL) under agreements FA9453-18-2-0022 and FA9550-22-1-0093.}
\thanks{Longsen Gao and Rafael Fierro are with the Electrical and Computer Engineering Department, The University of New Mexico, Albuquerque, NM 87131, USA. {\tt\small \{lgao1, rfierro\}@unm.edu}}
\thanks{Claus Danielson are with the Mechanical Engineering Department, The University of New Mexico, Albuquerque, NM 87131, USA.{\tt\small \{cdanielson\}@unm.edu}}
}

\begin{document}

\maketitle

\begin{abstract}
The challenge of satellite stabilization, particularly those with uncertain flexible dynamics, has become a pressing concern in control and robotics. These uncertainties, especially the dynamics of a third-party client satellite, significantly complicate the stabilization task. This paper introduces a novel adaptive detumbling method to handle non-rigid satellites with unknown motion dynamics (translation and rotation). The distinctive feature of our approach is that we model the non-rigid tumbling satellite as a two-link serial chain with unknown stiffness and damping in contrast to previous detumbling research works which consider the satellite a rigid body. We develop a novel adaptive robotics approach to detumble the satellite by using two space tugs as servicer despite the uncertain dynamics in the post-capture case. Notably, the stiffness properties and other physical parameters, including the mass and inertia of the two links, remain unknown to the servicer. Our proposed method addresses the challenges in detumbling tasks and paves the way for advanced manipulation of non-rigid satellites with uncertain dynamics.
\end{abstract}

\section{Introduction}\label{sec:intro}
Maintenance and repair of aging satellites is becoming increasingly important as near-earth space becomes more crowded. Often the first step in servicing is for the servicer satellite to gain custody and stabilize the damaged client satellite. The problem of a satellite self-stabilizing has been well-studied~\cite{kunciw1976optimal,bak1996autonomous,ward1997,kucharski2009}. Early efforts focused on passive methods on client satellites, such as magnetic coils and gravity-gradient booms, to stabilize satellites~\cite{bayat2009stabilization,ivanov2013hysteresis}, while modern satellites typically use reaction wheels and control moment gyroscopes for active self-stabilization~\cite{votel2012comparison,shen2018rigid,leve2015spacecraft}. However, if the passive or active systems malfunction then the client satellite will not be capable to self-stabilizing. This paper considers the challenging problem of using a robotic servicer satellite to grasp and stabilizing a malfunctioning client satellite which has been studied in \cite{gang2020detumbling,yan2021trajectory,chen2023optimal,down2023adaptive,danielson2024experimental}.

Satellite detumbling is a variation of the satellite stabilization problems. In the satellite detumbling problem, a predominant assumption is that the satellite's rotation is the primary concern, while the manipulator system and the satellite maintain a static relative translational position\cite{aghili2009time,luo2014survey,liu2019active,sugai2013detumbling}. While applicable in many conditions, this approach poses limitations when addressing more complex scenarios. Specifically, in situations where a satellite is malfunctioning, it is imperative to halt not only its rotational motion but to stabilize its orbit\cite{taylor2006orbital}. Failing to control both aspects can cause the client satellite to leave its designated orbital window. Thus, there's a pressing need to develop comprehensive control mechanisms for both rotational and translational detumbling. 

Another assumption in the majority of detumbling work is that the tumbling satellite is treated as a rigid body, which ignores the non-rigid dynamics of the satellite e.g. solar panels. Detumbling is challenging due to the non-rigid dynamics of the client satellite which are uncertain.

\begin{figure}[!t]
\centering
\includegraphics[width=2.6in]{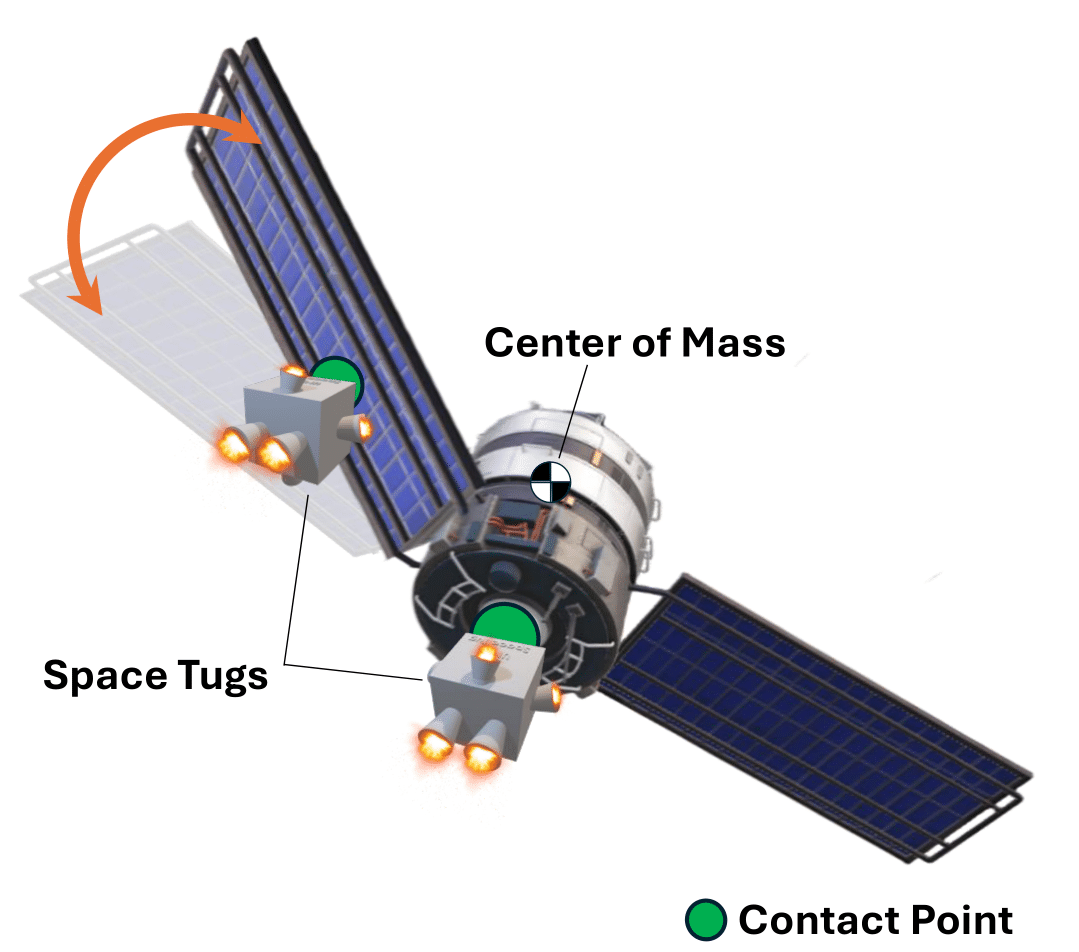}
\caption{Two space tugs collaboratively detumble a non-rigid satellite together in Space environment.}
\label{fig_1}
\end{figure}

Furthermore, when servicing a third-party satellite, the physical parameters will be highly uncertain.  For instance, the mass and inertia of the client satellite will be uncertain especially when the satellite has been operating for a long period leading to an expenditure of fuel\cite{clemen2002onorbit,culbertson2021decentralized,stolfi2018parametric}. Similarly, the stiffness properties of the satellite, which can vary due to factors such as temperature changes and material degradation, can affect the interaction dynamics between the servicer and the satellite \cite{nenarokomov2019environmental}. Finally, the points on the client satellite where the robotic servicer makes contact will be highly uncertain. These uncertainties have been a major bottleneck in developing reliable and robust detumbling methods.

This paper addresses the pressing need for adaptive detumbling methods that can handle the uncertainties of the non-rigid client satellite dynamics. In this paper, first we introduce the analysis of a non-rigid satellite which transforms a malfunctioning satellite model into a more manageable ``two-link chain" model. This transformation not only facilitates easier analysis of the non-rigid satellite's dynamics and stiffness but also forms the basis for our detumbling approach. Then, we propose an adaptive detumbling method applied on two space tugs to stop the motion of the tumbling satellite under those uncertainties. Fig.\ref{fig_1} shows our general concept using two space tugs  detumbling a non-rigid satellite cooperatively in the space environment.

The \textbf{contributions} of this paper are twofold: \textbf{(i)} We consider a more complex detumbling scenario in which the non-rigid client satellite is modeled as a two-link serial chain, and \textbf{(ii)} introduce an adaptive detumbling method that operates effectively under uncertainties, including unknown mass, inertia, stiffness properties and the location of contact points. This method represents a robust solution to the longstanding challenges associated with satellite detumbling.

\section{On-Orbit Detumbling Problem}\label{sec:method}
This section describes the detumbling problem for the client satellite, including the dynamics of the client satellite and the grasping dynamics from two space tugs.

\subsection{Agents and Reference Frames}
This paper proposes a novel methodology for detumbling a client satellite by using two space tugs. We consider two agents: the tumbling client satellite and the space tug, described below:

\textbf{Tumbling client satellite} The client tumbling satellite is composed of a base and two solar panels: one solar panel is connected to the base through a functioning hinge stability, which we treat as a ``fixed joint"; the other solar panel is also connected to the base but through a malfunctioning hinge, which we will model as a ``revolute joint" that has unknown dynamics. To simplify the non-rigid satellite model, here we convert the satellite into a ``two-link chain" model as following in detail:
\begin{itemize}
 \item \textbf{Link-1} As Fig. \ref{fig_2} shows, the client satellite is comprised of the base (small orange square with solid black outline) and the solar panel (bottom blue rectangle) and their functioning ``fixed joint" as the combined module, which is our defined ``\textit{Link-1}". The combined module (orange rectangle with black dash line) can be assumed as a normal rigid body with unknown physical properties e.g. mass, inertia.
 \item \textbf{Link-2} As Fig. \ref{fig_2} shows, the solar panel (upper blue rectangle) and hybrid hinge (yellow rectangle) are composed as our defined ``\textit{Link-2}". We assume related physical properties are also unknown, e.g. mass, inertia, and hinge stiffness (friction, damper, spring). 
\end{itemize}
Fig.\ref{fig_3} shows a simplified diagram of the ``two-chain link" model as we explained above.

\textbf{Space Tug}
A space tug is an independently operating, self-propelled spacecraft for assembly, maintenance, repair, and contingency operations in $SE(3)$~\cite{gao2023autonomous}. Based on our previous work, we upgraded our space tug structure design in this paper as shown in Fig.\ref{fig_1} to ensure each tug with 6 rockets in proper directions that can push or pull in 3D space to output the desired wrench to the external environment.

We define the world-frame, denoted as $\mathcal{O}$, located at a nearby space station which would be used for maintenance tasks. In addition, we define two body-fixed frames designated as $\alpha$ and $\beta$ located on the center-of-mass of \textit{Link-1} and \textit{Link-2}, respectively. Both body-fixed frames are located at the center of mass of their respective links. Here we set up two estimated points $P_1$ and $P_2$ on \textit{Link-1} and \textit{Link-2}, respectively, as shown on Fig.\ref{fig_2}. The two estimated points are used for each space tug to share common measurement information to estimate each link's unknown properties concerning their locations. Let $\mathbf{q}_b, \mathbf{q}_s \in \reals^6$ be the combined position and orientation for \textit{Link-1} and \textit{Link-2}, respectively. $\mathbf{q}_b = [\mathbf{p}_b^\top, \boldsymbol{\theta}_b^\top]^\top \in \reals^{6 \times 1}$ in which $\mathbf{p}_b \in \reals^3$ and $\boldsymbol{\theta}_b \in \reals^{3}$ denote the position and Euler angle along each axis of $P_1$ relative to frame $\mathcal{O}$, respectively, and $\mathbf{q}_s = [\mathbf{p}_s^\top, \boldsymbol{\theta}_s^\top]^\top \in \reals^{6 \times 1}$ in which $\mathbf{p}_s \in \reals^3$ and $\boldsymbol{\theta}_s \in \reals^{3}$ denote the position and Euler angle along each axis of $P_2$ relative to frame $\mathcal{O}$, respectively. 
Note that we assume that the two estimated points $P_1$ and $P_2$ are closed to each other and also closed to the revolute joint which the related revolution of \textit{Link-1} and \textit{Link-2} has a minor effect on the relative movement between $\mathbf{p}_b$ and $\mathbf{p}_s$. So, in this case, we assume $\mathbf{p}_b \approx \mathbf{p}_s$ and the term $\mathbf{q}_s$ can be replaced by $\mathbf{q}_b$ with another term which will be introduced in Section \ref{sec::sub-D} to keep only one state variable in the whole dynamics system and simplify the controller design and stability analysis in Section \ref{sec::method-go}.
 $\mathbf{d}_b$ and $\mathbf{d}_s$ denote the position of the estimated point contact frame $\alpha$ and $\beta$, respectively.  We use $\mathbf{^{\alpha}_{\mathit{o}}R}$ and $\mathbf{^{\beta}_{\mathit{o}}R}$ denote the orientation for \textit{Link-1} and \textit{Link-2} in frame $\mathcal{O}$, respectively. The grasping point which located on \textit{Link-1} at $\mathbf{\bar d}_1$ in frame $\alpha$; another grasping point which located on \textit{Link-2} at $\mathbf{\bar d}_2$ in frame $\beta$. 
\begin{figure}[!t]
\centering
\includegraphics[width=3.4in]{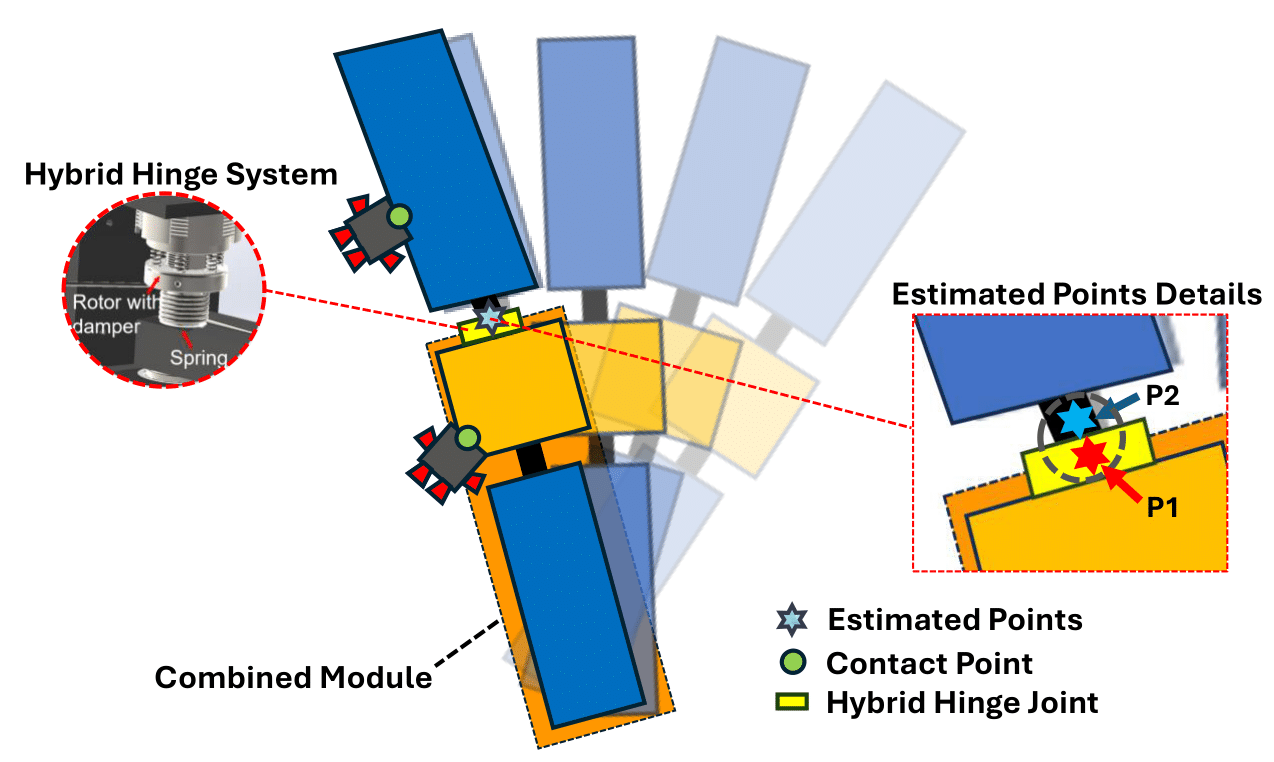}
\caption{Two free-flying space tugs de-tumble a client satellite. One space tug holds the base of the satellite; another space tug the solar panel, which is connected to the base by a hybrid stiffness system.}
\label{fig_2}
\end{figure}

\subsection{Dynamics of Link-1}
\label{dynamics-link1}
 To model the effect of the non-rigid part to the total wrench related the estimated point, we need to find the relative rotation between \textit{Link-1} and \textit{Link-2} as shown in Fig.\ref{fig_3}
\begin{equation}
 \mathbf{^{\beta}_{\alpha}R} = \mathbf{^{\beta}_{\mathit{o}}R} \mathbf{^{\alpha}_{\mathit{o}}R}^{-1},
 \label{eq:rule11}
\end{equation}
where $\mathbf{^{\beta}_{\alpha}R} \in \mathcal{SO} (3)$ denotes the relative rotation of $\textit{Link-2}$ in frame $\alpha$. Then we can 
define an operator as $(\cdot)^\ddagger$ that transfers a rotation  $\mathbf{R}\in \mathcal{SO} (3)$ into a quaternion $\mathfrak{q} \in \reals^4$ first and then to angle vector $\boldsymbol{\theta} \in \reals^3$ to avoid the gimbal lock. We can then
transfer (\ref{eq:rule11}) to get the relative rotation of \textit{Link-2} relative to frame $\alpha$ as
\begin{equation}
    \boldsymbol{\theta}_{*}^\top = \mathbf{^{\beta}_{\alpha}R}^\ddagger.
    \label{eq::rule1_2}
\end{equation}
Then we define a new variable as $\mathbf{q}_* = [\mathbf{p}_b^\top, \boldsymbol{\theta}_{*}^\top]^\top \in \reals^{6 \times 1}$ for the non-rigid part in which the torque applied by spring and damper directly depends on $\boldsymbol{\theta}_*$ instead of $\boldsymbol{\theta}$. 

The dynamics of the combined module on the client satellite with initial rotation and  translation can be modeled using the Euler-Lagrange equation in task space as 
\begin{equation}
 \begin{aligned}
 \resizebox{0.90\hsize}{!}{$
 \boldsymbol{\mathcal{T}}_{L1} = \mathbf{M}_b (\mathbf{q}_b)\mathbf{\ddot{q}}_b + \mathbf{C}_b (\mathbf{q}_b, \mathbf{\dot{q}}_b)\mathbf{\dot{q}}_b + \mathbf{D} (\mathbf{\dot{q}}_*) \mathbf{\dot{q}}_* 
 + \mathbf{K}(\mathbf{q}_*)\mathbf{q}_* + \mathbf{F}_f, \label{eq:L1} $}
 \end{aligned} 
\end{equation}
where $\boldsymbol{\mathcal{T}}_{L1} = \begin{bmatrix}
 \mathbf{f}_{L1},\boldsymbol{\tau}_{L1}
\end{bmatrix}^\top \in \reals^6$ is comprised of both the force $\mathbf{f}_{L1} = \begin{bmatrix}
 f_x, f_y, f_z
\end{bmatrix}^\top \in \reals^3$ and torque $\boldsymbol{\tau}_{L1} = \begin{bmatrix}
 \tau_x, \tau_y, \tau_z
\end{bmatrix}^\top \in \reals^3$ relative to the estimated point ${P_1}$. The inertia matrix of the \textit{Link-1} is
\begin{equation*}
 \begin{aligned}
 \mathbf{M}_b (\mathbf{q}_b) = \begin{bmatrix}
 m_b\mathbf{I} & -m_b (\mathbf{^{\alpha}_{\mathit{o}}R}\mathbf{d}_b)^\times\\
 -m_b (\mathbf{^{\alpha}_{\mathit{o}}R}\mathbf{d}_b)^\times & ^{\alpha}_{\mathit{o}}\mathbf{R}\boldsymbol{I}_{d_b} \mathbf{^{\alpha}_{\mathit{o}}R}^\top 
 \end{bmatrix},
 \end{aligned}
\end{equation*}
where $\boldsymbol{\mathrm{I}} \in \mathbb{R}^{3 \times 3}$ is identity matrix, $\mathbf{d}_b \in \reals^3$ and $\boldsymbol{I}_d$ are respectively the position vector of the measurement point concerning the center-of-mass and moment-of-inertia of the client in frame $\alpha$. Note that $(\cdot)^\times$ denotes pseudo cross-product operator for $\reals^3 \rightarrow \boldsymbol{so}(3)$. Moment of inertia is obtained using the parallel axis theorem\cite{khalil2002nonlinear}
\begin{equation*}
 \begin{aligned}
 \boldsymbol{I}_{d_b} = \boldsymbol{I}_{cm}^{L1} + m_b ( (\boldsymbol{d}_b^\top\boldsymbol{d}_b)\mathbf{I} - \boldsymbol{d}_b\boldsymbol{d}_b^\top), 
\end{aligned} 
\end{equation*}
where $\boldsymbol{I}_{cm}^{L1}$ is the moment-of-inertia about the center-of-mass of \textit{Link-1}. The centripetal and Coriolis matrix is
\begin{equation*}
\resizebox{0.98\hsize}{!}{$
 \mathbf{C}_b (\mathbf{q}_b, \mathbf{\dot{q}}_b)=\begin{bmatrix}
 \mathbf{0}_{3 \times 3} & -m_b \boldsymbol{\omega}^{\times}\left (^{\alpha}_{\mathit{o}}\mathbf{R} \mathbf{d}_{b}\right)^{\times} \\
-m_b \boldsymbol{\omega}^{\times}\left (\mathbf{^{\alpha}_{\mathit{o}}\mathbf{R} \mathbf{d}_b}\right)^{\times} & \boldsymbol{\omega}^{\times}\mathbf{^{\alpha}_{\mathit{o}}R} \boldsymbol{I}_{d_b}\mathbf{^{\alpha}_{\mathit{o}}R}^\top-m_b\left (\left (\mathbf{^{\alpha}_{\mathit{o}}R} \mathbf{d}_b\right)^{\times} {\mathbf{v}_b}\right)^{\times}
 \end{bmatrix}$},
\end{equation*}
where $\mathbf{v}_b = \mathbf{\dot{p}}_b \in \reals^3$ denote the linear velocity of the \textit{Link-1}. The servicer satellite will not have knowledge of the parameters $m_b$, $\boldsymbol{I}_{d_b}$, and $\mathbf{d}_b$ that define the dynamics (\ref{eq:L1}) of \textit{Link-1}. The damping matrix $\mathbf{D}(\mathbf{\dot{q}}_b)$ is
\begin{align}
 \mathbf{D} = \begin{bmatrix}
 \mathbf{0}_{3 \times 3} & \mathbf{0}_{3 \times 3}\\
 \mathbf{0}_{3 \times 3} & \operatorname{diag}(\boldsymbol{\xi})(\boldsymbol{\dot\theta}_*\oslash\boldsymbol{\dot \theta}_b)^\times 
 \end{bmatrix},
\end{align}
where $\oslash$ denotes element-wise division. $\boldsymbol{\xi} = \begin{bmatrix}
 \xi_x & \xi_y & \xi_z
\end{bmatrix}^\top \in \reals^3$ is the damping ratio separated along each rotation axis. And the spring matrix $\boldsymbol{K}(\boldsymbol{q})$ is
\begin{align}
 \mathbf{K} = \begin{bmatrix}
 \mathbf{0}_{3 \times 3} & \mathbf{0}_{3 \times 3}\\
 \mathbf{0}_{3 \times 3} & \operatorname{diag}(\mathbf{k})(\boldsymbol{\theta}_*\oslash\boldsymbol{\theta}_b)^\times
 \end{bmatrix},
\end{align}
$\mathbf{k} = \begin{bmatrix}
 k_x & k_y & k_z
\end{bmatrix}^\top \in \reals^3$ is the spring value separated along each rotation axis. And we assume there is a constant friction wrench $\mathbf{F}_f \in \reals^6$ exists on the hinge of the satellite when $\boldsymbol{\dot \theta}_* \neq \mathbf{0}$. Note that $\lambda \in (0,1)$ is a constant value that can separate the wrench applied by the hinge to \textit{Link-1} relative to estimated point $P_1$. This constant value can be canceled on the composite dynamics analysis in Section~\ref{grasp_subsec}
\subsection{Dynamics of Link-2}
\label{sec::sub-D}
Next, we analyze the non-rigid dynamics of a solar panel connected to the base through the hinge. The dynamics of hybrid stiffness system as
\begin{align}
\resizebox{0.90\hsize}{!}{$
 \boldsymbol{\mathcal{T}}_{L2} = \mathbf{M}_s (\mathbf{q}_s)\mathbf{\ddot{q}}_s + \mathbf{C}_s (\mathbf{q}_s, \mathbf{\dot{q}}_s)\mathbf{\dot{q}}_s + \mathbf{D} (\mathbf{\dot{q}}_*) \mathbf{\dot{q}}_* 
 + \mathbf{K} (\mathbf{q}_*)\mathbf{q}_* + \mathbf{F}_f, \label{eq:L2} $} 
\end{align}
where $ \boldsymbol{\mathcal{T}}_{L2} = \begin{bmatrix}
 \mathbf{f}_{L2},\boldsymbol{\tau}_{L2}
\end{bmatrix}^\top \in \reals^{6 \times 1}$ is the total wrench applied to the \textit{Link-2} relative the estimated point ${P_2}$. To keep the state variable uniform in the whole dynamics system to simplify the controller design, we can use $\mathbf{q} = \mathbf{q}_b$ to replace $\mathbf{q}_s, \mathbf{q_*}$ and reorganize (\ref{eq:L2}) as
\begin{align}
\resizebox{0.90\hsize}{!}{$
 \boldsymbol{\mathcal{T}}_{L2} = \mathbf{M}_s (\mathbf{q})\mathbf{\ddot{q}} + \mathbf{C}_s (\mathbf{q}, \mathbf{\dot{q}})\mathbf{\dot{q}} + \mathbf{D} (\mathbf{\dot{q}}) \mathbf{\dot{q}} 
 + \mathbf{K} (\mathbf{q})\mathbf{q} + \mathbf{F}_f, \label{eq:L2_new} $} 
\end{align}
where $\mathbf{M}_s$ is \textit{Link-2}'s inertia matrix as
\begin{equation*}
\resizebox{0.95\hsize}{!}{$
 \mathbf{M}_s (\mathbf{q}) 
 = \begin{bmatrix}
 m_s\mathbf{I} - m_s (\mathbf{^{\beta}_{\mathit{o}}R}\mathbf{d}_s)^\times(\boldsymbol{\ddot\theta}_*\oslash\mathbf{\ddot p}_b)^\times & -m_s (\mathbf{^{\beta}_{\mathit{o}}R}\mathbf{d}_s)^\times \\
 -m_s (\mathbf{^{\beta}_{\mathit{o}}R}\mathbf{d}_s)^\times & \mathbf{^{\beta}_{\mathit{o}}R}\boldsymbol{I_{d_s}} \mathbf{^{\beta}_{\mathit{o}}R}^\top - (\boldsymbol{\ddot\theta}_*\oslash\boldsymbol{\ddot \theta}_b)^\times 
 \end{bmatrix}$},
\end{equation*}
where $\boldsymbol{I}_{d_s}$ is the moment of inertia of the \textit{Link-2} as
\begin{align*}
 \boldsymbol{I}_{d_s} = \boldsymbol{I}_{rot} + \boldsymbol{I}_{cm}^{L2} + m_s ( (\boldsymbol{d}_{s}^\top \mathbf{d}_{s})\mathbf{I} - \mathbf{d}_{s} \mathbf{d}_{s}^\top ), \label{eq22}
\end{align*}
where $\boldsymbol{I}_{rot} \in \mathbb{R}^{3 \times 3}$ is the inertia matrix of the rotor. The centripetal and Coriolis matrix is
\begin{equation*}
\resizebox{1\hsize}{!}{$
 \mathbf{C}_s (\mathbf{q}, \mathbf{\dot{q}})= 
\begin{bmatrix}
 m_s \boldsymbol{\omega}^{\times}\left (\mathbf{^{\beta}_{\mathit{o}}R} \mathbf{d}_s\right)^{\times}(\boldsymbol{\dot\theta}_*\oslash\mathbf{\dot p}_b)^\times & -m_s \boldsymbol{\omega}^{\times}\left (\mathbf{^{\beta}_{\mathit{o}}R} \mathbf{d}_s\right)^{\times}\\
 -m_s \boldsymbol{\omega}^{\times}\left (\mathbf{^{\beta}_{\mathit{o}}R} \mathbf{d}_s\right)^{\times} & \boldsymbol{\omega}^{\times}\mathbf{^{\beta}_{\mathit{o}}R} \boldsymbol{I}_{d_s}\mathbf{^{\beta}_{\mathit{o}}R}^\top-m_s\left (\left (\mathbf{^{\beta}_{\mathit{o}}R} \mathbf{d}_s\right)^{\times} {\mathbf{v}_b}\right)^{\times}-(\boldsymbol{\dot\theta}_*\oslash\boldsymbol{\dot \theta}_b)^\times
\end{bmatrix}$},
\end{equation*}
 $\mathbf{D}, \mathbf{K}$ and $\mathbf{F}_f$ are the same as in Section~\ref{dynamics-link1}.

\begin{figure}[!t]
\centering
\includegraphics[width=2.3in]{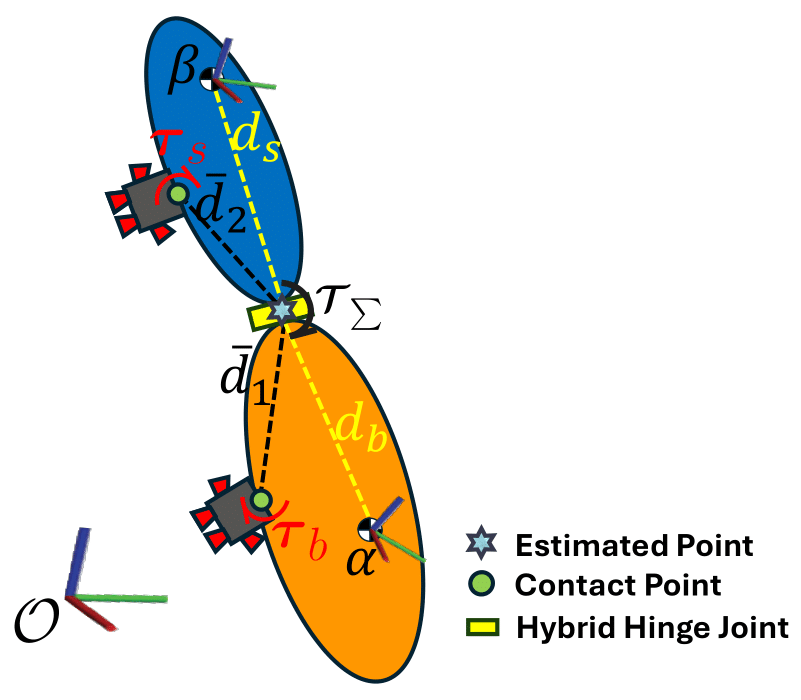}
\caption{Simplified diagram of the detumbling process. The base and a fully functional solar panel connected to the base by a fixed joint together as a combined ``module" as \textit{Link-1}(orange ellipse); another nonfunctional solar panel as \textit{Link-2}(blue ellipse). The hybrid stiffness system can be treated as a specific revolute joint(yellow rectangle).}
\label{fig_3}
\end{figure}

\subsection{Grasping Dynamics of Manipulator System}
\label{grasp_subsec}
Next, we model the wrench $\boldsymbol{\mathcal{T}_b}$ and $\boldsymbol{\mathcal{T}_s}$ applied to the $\textit{Link-1}$ and $\textit{Link-2}$, respectively. 
The total wrench applied on the $\textit{Link-1}$ and $\textit{Link-2}$ can be represent as
\begin{equation}
 \begin{aligned}
\boldsymbol{\mathcal{T}}_{L1} &= \boldsymbol{G}_b\left (\mathbf{q}_b, \mathbf{\bar d}_1\right) \boldsymbol{\mathcal{T}}_b,\\
\boldsymbol{\mathcal{T}}_{L2} &= \boldsymbol{G}_s\left (\mathbf{q}_s, \mathbf{\bar d}_2\right) \boldsymbol{\mathcal{T}}_s,
\end{aligned}
\label{eq::grasping-dynamics}
\end{equation}
where the matrix $\boldsymbol{G}_{b,s}$ maps the forces and torques produced by the space tug into wrenches on the $\textit{Link-1}$ and $\textit{Link-2}$ 
\begin{equation*}
 \begin{aligned}
 \boldsymbol{G}_{b,s}\left (\boldsymbol{q}, \boldsymbol{\bar d}_i\right) = \begin{bmatrix}
 \boldsymbol{\mathrm{I}} & \mathbf{0}_{3 \times 3}\\
  (\mathbf{^{\alpha, \beta}_{\mathit{o}}R}\mathbf{\bar d}_i)^{\times} & \mathbf{I}
\end{bmatrix}.
\end{aligned}
\end{equation*}
Combined with the dynamics~\eqref{eq:L1} and~\eqref{eq:L2} and grasping dynamics in (\ref{eq::grasping-dynamics}), we obtain
\begin{equation*}
\left\{ 
\resizebox{0.97\hsize}{!}{$
\begin{aligned}
\boldsymbol{G}_b\left (\mathbf{q}_b, \mathbf{\bar d}_1\right)\boldsymbol{\mathcal{T}}_b&=\mathbf{M}_b (\mathbf{q})\mathbf{\ddot{q}} + \mathbf{C}_b (\mathbf{q}, \mathbf{\dot{q}})\mathbf{\dot{q}}+ \lambda(\mathbf{D} (\mathbf{\dot{q}})\mathbf{\dot{q}} 
 + \mathbf{K} (\mathbf{q}) + \mathbf{F}_f) \\
\boldsymbol{G}_s\left (\mathbf{q}_s, \mathbf{\bar d}_2\right) \boldsymbol{\mathcal{T}}_s &=
 \mathbf{M}_s (\mathbf{q})\mathbf{\ddot{q}} + \mathbf{C}_s (\mathbf{q}, \mathbf{\dot{q}})\mathbf{\dot{q}} + \mathbf{D}(\mathbf{\dot{q}})\mathbf{\dot{q}} 
 + \mathbf{K} (\mathbf{q}) + \mathbf{F}_f.
\end{aligned}$}
\right.
\end{equation*}
Then we can get the composite wrench $\boldsymbol{\mathcal{T}}_{\sum}$ applied on the hinge relative to the estimated point
\begin{equation}
\resizebox{0.85\hsize}{!}{$
\begin{aligned}
 \boldsymbol{\mathcal{T}}_{\sum} &= \boldsymbol{\mathcal{T}}_{L1} + \boldsymbol{\mathcal{T}}_{L2}\\
 &=\mathbf{M}_\nu(\mathbf{q})\mathbf{\ddot{q}} + \mathbf{C}_\nu (\mathbf{q}, \mathbf{\dot{q}})\mathbf{\dot{q}} + \mathbf{D}_\nu (\mathbf{\dot{q}})\mathbf{\dot{q}} 
 + \mathbf{K}_\nu (\mathbf{q}) + \mathbf{F}_\nu,
 \end{aligned}$}
 \label{eq:all_wrench}
\end{equation}
where $\mathbf{M}_\nu = \mathbf{M}_b + \mathbf{M}_s$ and $\mathbf{C}_\nu = \mathbf{C}_b + \mathbf{C}_s$ are combined inertia matrix and Coriolis and centripetal matrix, respectively, for both \textit{Link-1} and \textit{Link-2} relative to the estimated point. $\mathbf{D}_\nu = 2\mathbf{D}$, $\mathbf{K}_\nu = 2\mathbf{K}$ and $\mathbf{F}_v = 2\mathbf{F}_f$ based on the Newton's third law of motion\cite{khalil2002nonlinear}. Proof of the positive definiteness of the mass matrix is included in Appendix~\ref{appendixproof1}.

\subsection{Satellite Stabilization Problem}
The client satellite stabilization problem can be stated as follows:

\textit{Problem 1:} Drive the \textit{Link}-1 and \textit{Link}-2 with initial condition $\mathbf{\dot q}_b(0),\mathbf{\dot q}_s(0) \neq \mathbf{0}^{6\times1}$ to zero as $\mathbf{\dot q}_b(t) = \left[\mathbf{v}_b(t), \boldsymbol{\omega}_b(t)\right]^\top \rightarrow \mathbf{0}\in\reals^6$ and $\mathbf{\dot q}_s(t) = \left[\mathbf{v}_s(t), \boldsymbol{\omega}_s(t)\right]^\top \rightarrow \mathbf{0}\in\reals^6$ as $t \rightarrow \infty$.

According to (\ref{eq:rule11}), when the $\boldsymbol{q} \rightarrow \mathbf{0}$, we obtain that the hinge angle converges to zero meaning that the solar panel is aligned with the client satellite body. Here we introduce a composite error $\boldsymbol{s}$ that combines both linear and angular velocity error to help us make sure the motion stabilization: 
\begin{equation}
\label{eq:error}
 \begin{aligned}
 \boldsymbol{s} = \begin{bmatrix}
 \boldsymbol{\epsilon} \\
 \boldsymbol{o}
\end{bmatrix} = \begin{bmatrix}
 \dot{\boldsymbol{e}}_p + \gamma \boldsymbol{e}_p (t)\\
 \boldsymbol{e}_{w} + \mathcal{S}_{e} 
\end{bmatrix},\\
\end{aligned}
\end{equation}
where $\mathbf{e}_p (t) = \mathbf{p}_d(t) - \mathbf{p}(t)$ is the position error, $\gamma \in (0,1)$ is a constant positive number. $\boldsymbol{\dot{e}}_p (t)$ is the linear velocity error; $\boldsymbol{e}_{w} = \boldsymbol{\omega}_d (t) - \boldsymbol{\omega} (t)$ denote the angular velocity error on estimated point. For 
\begin{equation*}
 \begin{aligned}
 \mathcal{S}_{e} &= \tfrac{1}{2}\gamma\mathbf{^{\alpha}_{\mathit{o}}R}_d \left (\mathbf{R}_e + \mathbf{R}_e^\top\right)^{\vee},
 \end{aligned}
\end{equation*}
where $\mathcal{S}_e \in \reals^6$ is the part of the rotation dynamics error in the skew-symmetric form. And $\mathbf{R}_e = \mathbf{^{\alpha}_{\mathit{o}}R}_d^\top\mathbf{^{\alpha}_{\mathit{o}}R}$ is the rotation error. Note that our goal for $\mathbf{R}_e \rightarrow \mathbf{I}$, $\boldsymbol{\dot e}_p \rightarrow \mathbf{0}$, $\boldsymbol{e}_\omega \rightarrow \mathbf{0}$ as $t \rightarrow \infty$.

\textit{Problem 1} is thus equivalent to ensuring the convergence of the composited error $\boldsymbol{s} \rightarrow \mathbf{0}\in \reals^6$ as $t \rightarrow \infty$.

\section{Decentralized Adaptive Stabilization}
\label{sec::method-go}
In this section, we will design a decentralized adaptive controller for the two space tugs on the servicer satellite to manipulate $\textit{Link-}1$ and $\textit{Link-}2$ to drive the client satellite to rest, i.e., $\boldsymbol{s} \rightarrow \mathbf{0}^{6 \times 1}$ as $t \rightarrow \infty$.  Then, we define a Lyapunov-like function\cite{blanchini2008set} to prove that our controller design makes the system stable.

First, let's introduce the regressor matrix $\boldsymbol{Y}_{\varphi}^{\nu} (\mathbf{q},\mathbf{\dot{q}}, \mathbf{\dot{q}}_d, \mathbf{\ddot{q}}_d)$ which is necessary for the controller adaptation
\begin{equation}
\boldsymbol{Y}_{\varphi}^{\nu} = \left[{\boldsymbol{Y}_\epsilon^{\nu}}^\top, {\boldsymbol{Y}_o^\nu}^\top  \right]^\top,\\
\label{eq:regressor_def} 
\end{equation}
where $\boldsymbol{Y}_\epsilon^\nu, \boldsymbol{Y}_o^\nu \in \reals^{3\times22}$ are sub-regressor matrices. $\boldsymbol{\epsilon},\boldsymbol{o} \in \reals^{3}$ in (\ref{eq:error}) denote recursive terms for angular and linear acceleration relative to the estimated point, respectively.
 
 We decompose the composite dynamics (\ref{eq:all_wrench}) into the regressor matrix $\boldsymbol{Y}_{\varphi}^\nu$ and a parameter vector $ \boldsymbol{\varphi} \in \reals^{22}$ which contain all the unknown physical parameters for the whole client satellite as summarized in Table~\ref{tab:task}. Then, based on (\ref{eq:error}) and (\ref{eq:regressor_def}) and the dynamics model in (\ref{eq:all_wrench}), we can get
\begin{equation}
\begin{aligned}
\boldsymbol{Y}_{\varphi}^\nu\boldsymbol{\tilde{\varphi}} &= \sum_{i=1}^2 \mathit{c} (\mathbf{\tilde{M}}_\nu \mathbf{\ddot{{q}}}_{d} + \mathbf{\tilde{C}}_\nu\mathbf{\dot{q}}_{d} + \mathbf{\tilde D}_\nu \mathbf{\dot{q}}_d 
 + \mathbf{\tilde K}_\nu\mathbf{q}_d+ \mathbf{\tilde F}_\nu)\\
 &= \sum_{i=1}^2 \mathit{c} (\mathbf{\tilde M}_\nu\mathbf{\ddot{q}}_d + \boldsymbol{\tilde \Lambda}\mathbf{\dot{q}}_d + \mathbf{\tilde K}_\nu \mathbf{q}_d + \mathbf{\tilde F}_\nu),
 \label{eq:regressor-phi}
\end{aligned}
\end{equation}
where $\mathit{c}=\tfrac{1}{2}$ divides the control workload evenly among the space tugs, and   
 $\boldsymbol{\tilde{\varphi}}= \boldsymbol{\hat{\varphi}} - \boldsymbol{\varphi}$ denotes the error of parameters estimation related to the physical properties of the \textit{Link-1} and \textit{Link-2}, scaled by $\mathit{c}$, and $\mathbf{\tilde{M}_\nu (\mathbf{q})} = \mathbf{\hat{M}}_\nu (\mathbf{q}) - \mathbf{M}_\nu (\mathbf{q})$ is the inertia error of the stiffness system, $\mathbf{\tilde{C}}_\nu (\mathbf{\dot{q}}, \mathbf{q}) = \mathbf{\hat{C}}_\nu (\mathbf{\dot{q}}, \mathbf{q}) - \mathbf{C}_\nu (\mathbf{\dot{q}}, \mathbf{q})$ is the Coriolis and centripetal error and $\mathbf{\tilde F}_\nu = \mathbf{\hat F}_\nu - \mathbf{F}_\nu$ is the friction error. $\boldsymbol{\tilde \Lambda} = \mathbf{\tilde C}_\nu + \mathbf{\tilde D}_\nu$ is the combined matrix to simplify the equation.

To provide detumbling despite parametric uncertainty, we use the proportional-derivative control law with adaptive feedback linearization
\begin{equation}
\begin{aligned}
 \boldsymbol{\mathcal{\hat T}} = \boldsymbol{Y}^\nu_{\varphi} \boldsymbol{\hat{\varphi}} - \boldsymbol{\mathcal{K}}_{\scriptstyle\text{PD}}\boldsymbol{s},
 \label{eq: pdlaw}
\end{aligned}
\end{equation}
where the term $\boldsymbol{Y}^\nu_{\varphi} \boldsymbol{\hat{\varphi}}$ feedback linearizes the nonlinear parts of the tumbling satellite dynamics~\eqref{eq:L1}-\eqref{eq:L2} using the estimated parameters $\boldsymbol{\hat{\varphi}}$. The term $-\boldsymbol{\mathcal{K}}_{\scriptstyle\text{PD}}\boldsymbol{s}$ provide proportional-derivative control. 
Note that the composite error~\eqref{eq:error} contains both position and rotation errors as well as velocity and angular velocity errors, producing a proportional-derivative controller. 

Now let's introduce another set of regressors $\boldsymbol{Y}_d^\nu (\boldsymbol{\mathcal{\hat{T}}}_{\nu}, \mathbf{q})$ 
\begin{equation}
        \boldsymbol{Y}_d^\nu (\boldsymbol{\mathcal{\hat{T}}}_\nu, \mathbf{q}) = \left[\boldsymbol{Y}_d^b (\boldsymbol{\mathcal{\hat{T}}}_b, \mathbf{q}), \boldsymbol{Y}_d^s (\boldsymbol{\mathcal{\hat{T}}}_s, \mathbf{q})\right],
\end{equation}
where $\boldsymbol{Y}_d^b, \boldsymbol{Y}_d^s \in \reals^{6\times3}$ are sub-regressor matrix concerning each grasping point as 
\begin{equation*}
    \begin{aligned}
    \boldsymbol{Y}_d^b (\boldsymbol{\mathcal{\hat{T}}}_b, \mathbf{q})  &= \left[\mathbf{f}_{L1}^{\times}\mathbf{^{\alpha}_{\mathit{o}}R},\boldsymbol{\tau}_{L1}^{\times}\mathbf{^{\alpha}_{\mathit{o}}R}\right]^\top\\
    \boldsymbol{Y}_d^s (\boldsymbol{\mathcal{\hat{T}}}_s, \mathbf{q})  &= \left[\mathbf{f}_{L1}^{\times}\mathbf{^{\beta}_{\mathit{o}}R},\boldsymbol{\tau}_{L1}^{\times}\mathbf{^{\beta}_{\mathit{o}}R}\right]^\top,
    \end{aligned}
\end{equation*}
then we can get
\begin{equation}
\begin{aligned}
\boldsymbol{Y}_d^\nu (\boldsymbol{\mathcal{\hat{T}}}_{\nu}, \mathbf{q})\mathbf{\hat{d}} &= -(\boldsymbol{\tilde{G}}_b \boldsymbol{\mathcal{\hat{T}}}_{b} + \boldsymbol{\tilde{G}}_s \boldsymbol{\mathcal{\hat{T}}}_{s}).
\label{eq:regressor-g}
\end{aligned}
\end{equation}
Thus, we can get the adaptation laws as
\begin{subequations}
\label{eq:adaptation}
\begin{align}
 \boldsymbol{\dot{\hat{\varphi}}} & =-\boldsymbol{\Gamma}_{\varphi} \boldsymbol{Y}_{\varphi}^{\nu}\left (\mathbf{q}, \mathbf{\dot{q}}, \mathbf{\dot{q}}_{d}, \ddot{\mathbf{q}}_{d}\right)^\top\boldsymbol{s}, \\
 \mathbf{\dot{\hat{d}}} & =-\boldsymbol{\Gamma}_d \boldsymbol{Y}_d^{\nu}\left (\boldsymbol{\mathcal{\hat{T}}}_{b,s}, \mathbf{q}\right)^\top\boldsymbol{s},
\end{align}
\end{subequations}
where $\boldsymbol{\Gamma}_{\varphi}  \in \reals^{22\times22}$ and $\boldsymbol{\Gamma}_{d} \in \reals^{6\times6}$ are symmetric positive definite matrices related to the adaptive gain, usually diagonal.

\textit{Theorem 1:} The composite error (\ref{eq:error}) of the composite dynamics for \textit{Link-1} and \textit{Link-2} relative to the estimated point in (\ref{eq:all_wrench}) converges to zero $\boldsymbol{s} \rightarrow \mathbf{0} \in \mathbb{R}^6$ as $t \rightarrow \infty$ under the adaptive controller (\ref{eq: pdlaw}) with the adaption law (\ref{eq:adaptation}).

\begin{proof} 
Consider the Lyapunov-like function 
\begin{equation}
 V (t) = \frac{1}{2}\left[\boldsymbol{s}^\top \mathbf{M}_\nu \boldsymbol{s} + \tilde{\boldsymbol{\varphi}}^\top \boldsymbol{\Gamma}_{\varphi} \tilde{\boldsymbol{\varphi}} +\mathbf{\tilde{d}}^\top \boldsymbol{\Gamma}_d \mathbf{\tilde{d}} \right] \label{eq9}
\end{equation}
where $\mathbf{\tilde{d}}$ are the grasping position error which are related the difference between the actual position of $i$-th space tug $\mathbf{\bar d}_i$ concerning the measurement point and the space tug's estimated $\mathbf{\hat{\bar d}}_i$. 
Then, by taking the derivative, we obtain
\begin{equation}
\begin{aligned}
\dot V (t) &=\boldsymbol{s}^\top \left[ \boldsymbol{\mathcal{T}}_{\sum} - \mathbf{M}_\nu \mathbf{\ddot{q}} - \boldsymbol{\Lambda}\mathbf{\dot{q}} - \mathbf{K}_\nu \mathbf{q} - \mathbf{F}_\nu \right]\\
 &+ \tilde{\boldsymbol{\varphi}}^\top \boldsymbol{\Gamma}_{\varphi} \dot{\hat{\boldsymbol{\varphi}}}+\mathbf{\tilde{d}}^\top \boldsymbol{\Gamma}_d \mathbf{\dot{\hat{d}}}, \label{eq10} 
\end{aligned}
\end{equation}
where $\dot{\hat{\boldsymbol{\varphi}}} = \dot{\tilde{\boldsymbol{\varphi}}}$ and $\mathbf{\dot{\hat{d}}} = \mathbf{\dot{\tilde{d}}}$ since the parameter $\boldsymbol{\varphi}$ and $\boldsymbol{d}$ are constant. Then we use the properties of skew-symmetry \cite{khalil2002nonlinear} to eliminate the term $ \boldsymbol{s}^\top (\tfrac{1}{2}\dot{\mathbf{M}}_{\nu} - \boldsymbol{\Lambda}) \boldsymbol{s} = 0$ and $\mathbf{\dot{q}}_d = \boldsymbol{s} - \mathbf{\dot{q}}$.
Let's define the control-law as
\begin{align}
 \boldsymbol{\mathcal{T}} = \boldsymbol{\hat{M}}_\nu \mathbf{\ddot{q}}_d + \boldsymbol{\hat{\Lambda}}\mathbf{\dot{q}}_d + \mathbf{\hat{K}}_\nu \mathbf{q}_d + \mathbf{\hat{F}}_\nu. \label{eq11}
\end{align}
Then we substitute \eqref{eq11} to \eqref{eq10} 
 and then rewrite $\dot{V} (t)$ using \eqref{eq:regressor-phi} and \eqref{eq:regressor-g} 
 and then substitute adaptation-laws~\eqref{eq:adaptation} to obtain
\begin{equation}
\resizebox{0.89\hsize}{!}{$
 \begin{aligned}
 \begin{split}
 \dot{V} (t) &= -\boldsymbol{s}^\top\boldsymbol{\mathcal{K}}_{{\scriptstyle\text{PD}}}\boldsymbol{s} + \boldsymbol{\tilde{\varphi}}^\top (\boldsymbol{Y}_{\varphi}^\nu \boldsymbol{\tilde{\varphi}} 
 + \boldsymbol{\Gamma}_{\varphi}\dot{\hat{\boldsymbol{\varphi}}})+ \mathbf{\tilde d}^\top (\boldsymbol{Y}_d^\nu \mathbf{\tilde d} + \boldsymbol{\Gamma}_d \mathbf{\dot{\hat{d}}})\\
 &= -\left (\boldsymbol{s}^\top \boldsymbol{K}_{\scriptstyle\text{PD}} \boldsymbol{s} \right) \leq 0.\label{eq19} 
 \end{split}
 \end{aligned}$}
\end{equation}
Thus, we get $V (t)>0$ and $\dot{V} (t) \leq 0$, and therefore the system is stable in the sense of Lyapunov.
\end{proof}

\section{Simulation Results}
\label{sec:experiments}

In this section, we present simulation results for detumbling task. We consider two space tugs as our manipulator system to detumble a client satellite in the zero-gravity environment. We used the MuJoCo\cite{todorov2012mujoco} simulation platform and modeled the two space tugs and client satellite operating in a zero-gravity environment as shown in Fig. \ref{fig_4}. Table~\ref{tab:task} summarizes the parameters used in this simulation. Note that we assume all parameters in Table~\ref{tab:task} are unknown to the controller.
\begin{table}[]
\centering
 \begin{tabular}{ccc}
 \toprule 
  Definition & Parameter & Value\\
 \midrule
  base & $m_b$ & 40 kg\\
  & $I_{x x}$ & $2.5667$ kg$\cdot$m$^2$\\
  & $I_{y y}$ & $2.5667$ kg$\cdot$m$^2$\\
  & $I_{z z}$ & $1.6667$ kg$\cdot$m$^2$\\
  solar panel-1 & $m_{s1}$ & 10 kg\\
  & $I_{x x}$ & $0.614$ kg$\cdot$m$^2$\\
  & $I_{y y}$ & $2.2771$ kg$\cdot$m$^2$\\
  & $I_{z z}$ & $2.4833$ kg$\cdot$m$^2$\\
  solar panel-2 & $m_{s2}$ & 10 kg\\
  & $I_{x x}$ & $1.8104$ kg$\cdot$m$^2$\\
  & $I_{y y}$ & $3.4771$ kg$\cdot$m$^2$\\
  & $I_{z z}$ & $3.6833$ kg$\cdot$m$^2$\\
  hybrid hinge & $J_{x x}$ & $6$ kg$\cdot$m$^2$\\
  & $J_{y y}$ & $4$ kg$\cdot$m$^2$\\
  & $\mathbf{\xi}_x$ & $0.4$\\
  & $\mathbf{\xi}_y$ & $0.2$\\
  & $\mathbf{\kappa}_x$ & $0.5$\\
  & $\mathbf{\kappa}_y$ & $0.25$\\
  position & $\mathbf{d}_b$ & $[-1.2,0.2,-0.1]$ m\\
  & $\mathbf{d}_s$ & $[0.74,0.0.1,-0.2]$ m\\
  & $\mathbf{\bar d}_1$ & $[0.36,-0.13,-0.44]$ m\\
  & $\mathbf{\bar d}_2$ & $[-1.0,-0.25,-0.3]$ m\\
 \bottomrule
 \end{tabular}
 \caption {Simulation parameters for detumbling task.}
 \label{tab:task}
\end{table}

\begin{figure*}[!t]
\centering
\includegraphics[width=6.5in]{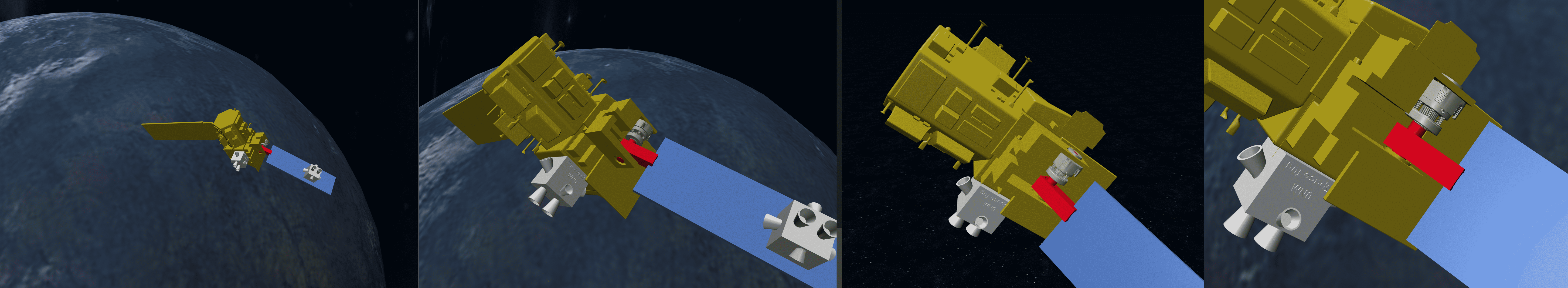}
\caption{Two Space tugs attaching on different locations of a satellite\protect\footnotemark[1] to apply wrenches on the  \textit{Link-1} (yellow part) and \textit{Link-2} (blue part) in zero-gravity simulation environment. We import the hybrid hinge system from real and integrate the friction, spring and damper properties on the rotor.  Supplemental video:
 \href{http://tiny.cc/stabglsfinal}{http://tiny.cc/stabglsfinal}}
\label{fig_4}
\end{figure*}
\footnotetext[1]{ICESat-2, developed by \href{https://icesat-2.gsfc.nasa.gov/}{NASA.}}

We consider the post-capture phase in this work. Thus, we assume both space tugs are firmly attached the satellite at fixed points. We set up the estimated points $P_1$ and $P_2$ in the hinge. For the desired linear velocity $\mathbf{v}_d = \left[ 
v_{x_d}, v_{y_d},v_{z_d} \right]^\top = \mathbf{0}\in \reals^{3}$ and desired angular velocity $\boldsymbol{\omega}_d = \left[ 
\omega_{x_d}, \omega_{y_d},\omega_{z_d} \right]^\top = \mathbf{0}\in \reals^{3}$. 

Considering the \textit{Link-1} and \textit{Link-2} as different parts of the client satellite with different initial angular rates
\begin{equation*}
 \boldsymbol{\omega}_{L1} (0) = \begin{bmatrix}
 0.7\\
 0.8\\
 -1.0
 \end{bmatrix} \text{rad/s},\ 
  \boldsymbol{\omega}_{L2} (0) = \begin{bmatrix}
 1.0\\
 0.5\\
 -1.7
 \end{bmatrix} \text{rad/s}.
\end{equation*} 

The initial linear rate of the satellite relative to the estimated point as 
\begin{equation*}
 \mathbf{v} (0) = \begin{bmatrix}
 -1.0\\
 0.9\\
 0.8
 \end{bmatrix} \text{m/s}.
\end{equation*}

\begin{figure}[htp]
\centering
\captionsetup[sub]{font=scriptsize,labelfont={sf},oneside,margin={0.75cm,0cm}}
\begin{subfigure}{3in}
 \includegraphics[width=\textwidth]{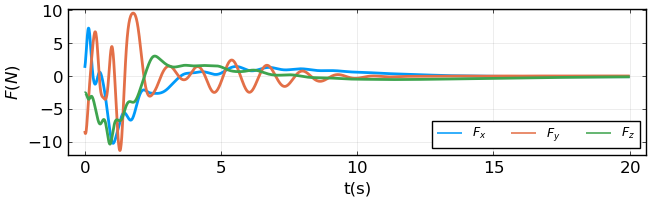}
 \caption{}
 \label{fig5_a}
\end{subfigure}
\hfill
\begin{subfigure}{3in}
 \includegraphics[width=\textwidth]{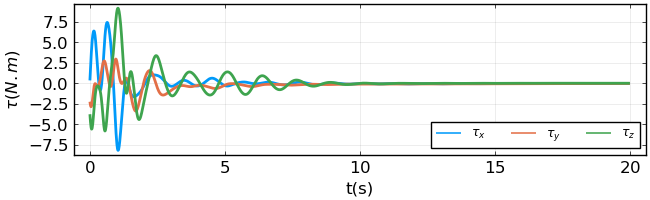}
 \caption{}
 \label{fig5_b}
\end{subfigure}
\caption{Force(a) and Torque(b) applied by space tug-1.}
\label{figure_5}
\end{figure}

\begin{figure}
\centering
\captionsetup[sub]{font=scriptsize,labelfont={sf},oneside,margin={0.65cm,0cm}}
\begin{subfigure}{3in}
 \includegraphics[width=\textwidth]{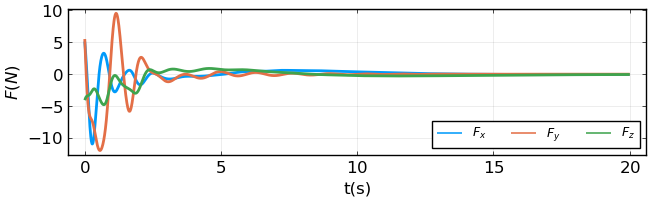}
 \caption{}
 \label{fig6_a}
\end{subfigure}
\begin{subfigure}{2.94in}
 \includegraphics[width=\textwidth]{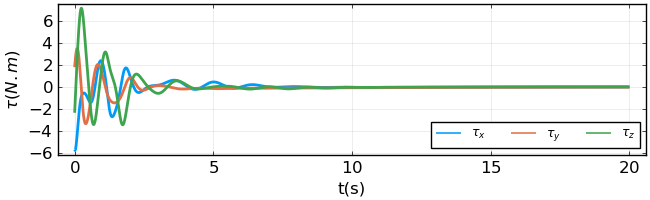}
 \caption{}
 \label{fig6_b}
\end{subfigure}
\caption{Force and Torque applied by space tug-2.}
\label{figure_6}
\end{figure}

\begin{figure}[htb]
\vspace{-0.07cm}
\centering
\captionsetup[sub]{font=scriptsize,labelfont={sf},oneside,margin={0.75cm,0cm}}
\begin{subfigure}{3in}
 \includegraphics[width=\textwidth]{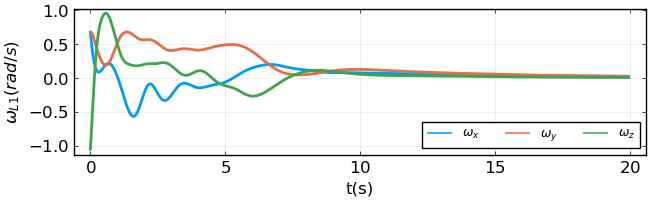}
 \caption{}
 \label{fig7_a}
\end{subfigure}
\hfill
\begin{subfigure}{3in}
 \includegraphics[width=\textwidth]{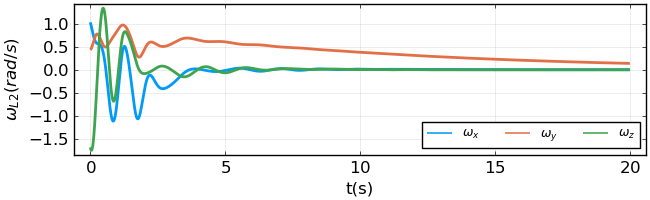}
 \caption{}
 \label{fig7_b}
\end{subfigure}
\hfill
\begin{subfigure}{3in}
 \includegraphics[width=\textwidth]{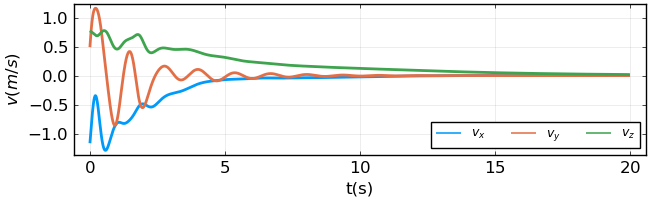}
 \caption{}
 \label{fig7_c}
\end{subfigure}
\caption{Angular velocity of \textit{Link-1} (a) and \textit{Link-2} (b). Linear velocity of the satellite concerning the estimated point (c).}
\label{figure_7}
\end{figure}

Fig. \ref{figure_5} shows the total wrench applied by the space tug-1, which holds the base of the client satellite that is part of \textit{Link-1}. Fig. \ref{fig5_a} shows the force change versus time and Fig. \ref{fig5_b} shows the torque change versus time. Both force and torque converge to zero as the client satellite is detumbled.

Fig. \ref{figure_6} shows the total wrench applied by the space tug-2, which holds the solar panel of the client satellite that is part of \textit{Link-2}. Fig. \ref{fig6_a} shows the force change versus time and Fig. \ref{fig6_b} shows the torque change versus time. Both force and torque also converge to zero as the client satellite is detumbled. Thus, the presented adaptive control scheme successfully detumbled the satellite

Fig. \ref{figure_7} empircially demonstrates that the presented adaptive controller (\ref{eq: pdlaw}) and (\ref{eq:adaptation}) solves Problem 1. Fig. \ref{fig7_a} shows the angular velocity versus time of \textit{Link-1} and Fig. \ref{fig7_b} shows the angular velocity versus time of \textit{Link-2}; Fig. \ref{fig7_c} shows the linear velocity of the satellite versus time concerning the estimated point. Both angular and linear velocities converge to zero over the simulation time period. 

Fig. \ref{fig_8} shows the Lyapunov-like function $V (t)$ versus time for the simulation results. As Fig. \ref{fig_8} shows, the Lyapunov-like function is positive $V (t) > 0$ and decreasing $\dot{V} (t) \leq 0$ as time increases. This empirically verifies that our system is asymptotically stable.
\begin{figure}[!t]
\vspace{0.5cm}
\centering
\includegraphics[width=3in]{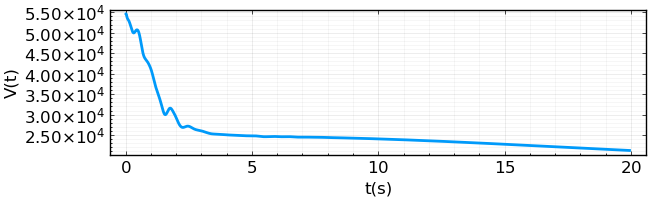}
\caption{Lyapunov-like function. }
\label{fig_8}
\end{figure}

Fig. \ref{figure_9} shows the parameter estimation of the whole dynamics system. Fig. \ref{fig_9a} shows the parameter estimates for the physical properties of the tumbling satellite including \textit{Link-1} and \textit{Link-2}; Fig. \ref{fig_9b} shows the all position estimate for both satellite and grasping dynamics. Based on the results, we can see that all signals are bounded and the final of values are close to the actual values.

Fig. \ref{fig11_a} and Fig. \ref{fig11_b} show the linear and angular velocity of the space tug-1; Fig. \ref{fig10_a} and Fig. \ref{fig10_b} show the linear and angular velocity of the space tug-2;  As Fig. \ref{figure_11} and Fig. \ref{figure_10} show, with the decreasing of the satellite's motion, both the motion of the two space tugs will also decrease. 
\vspace{-0.83cm}
\begin{figure}[!htb]
\centering
\captionsetup[sub]{font=scriptsize,labelfont={sf},oneside,margin={0.65cm,0cm}}
\begin{subfigure}{3in}
\vspace{0.125cm}
 \includegraphics[width=\textwidth]{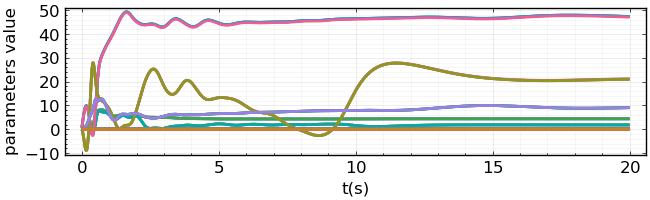}
 \caption{}
 \label{fig_9a}
\end{subfigure}
\hfill
\begin{subfigure}{2.95in}
 \includegraphics[width=\textwidth]{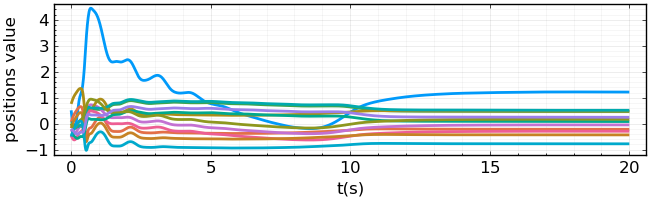}
 \caption{}
 \label{fig_9b}
\end{subfigure}
\caption{Physical parameters(a) and grasping positions(b) estimation.}
\label{figure_9}
\end{figure}\FloatBarrier
\vspace{-0.33cm}
\begin{figure}[!htb]
\vspace{0.222cm}
\centering
\captionsetup[sub]{font=scriptsize,labelfont={sf},oneside,margin={0.75cm,0cm}}
\begin{subfigure}{3in}
 \includegraphics[width=\textwidth]{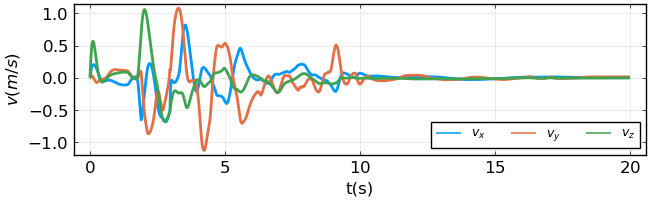}
 \caption{}
 \label{fig11_a}
\end{subfigure}
\hfill
\begin{subfigure}{3in}
 \includegraphics[width=\textwidth]{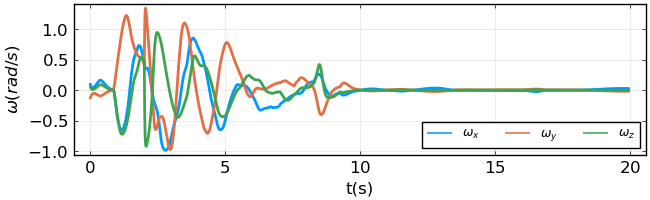}
 \caption{}
 \label{fig11_b}
\end{subfigure}
\caption{Linear (a) and angular (b) velocity of space tug-1.}
\label{figure_11}
\end{figure}
\begin{figure}[!htb]
\centering
\captionsetup[sub]{font=scriptsize,labelfont={sf},oneside,margin={0.75cm,0cm}}
\begin{subfigure}{3in}
 \includegraphics[width=\textwidth]{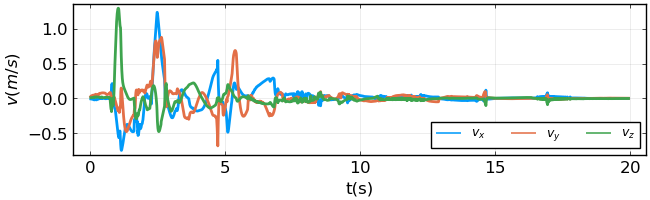}
 \caption{}
 \label{fig10_a}
\end{subfigure}
\hfill
\begin{subfigure}{3in}
 \includegraphics[width=\textwidth]{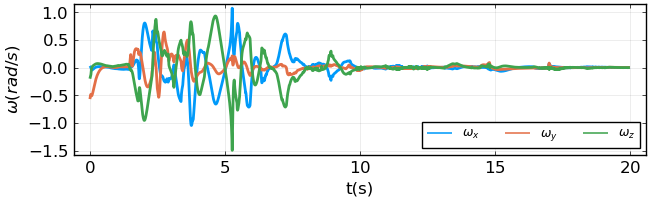}
 \caption{}
 \label{fig10_b}
\end{subfigure}
\caption{Linear (a) and angular (b) velocity of space tug-2.}
\label{figure_10}
\end{figure}
\section{Conclusions}\label{sec:conclusions}
This paper presented a decentralized adaptive detumbling controller for an on-orbit tumbling client satellite. We consider a more complicated malfunction on the tumbling satellite and transfer the whole system into a non-rigid model to analyze the dynamics in translation and rotation together. The controller we designed can drive the tumbling satellite to rest despite uncertainty about the system parameters and the location of the space tugs attached to the client. We presented simulation results demonstrating the efficacy of the cooperative adaptive robotic system.
\section*{Acknowledgement}\label{sec:acknoldege}
We thank Andrew Kwas from Northrop Grumman Corporation and Howie Choset from CMU for their discussions and advice on satellite detumbling. Also, we would like to thank Giovanni Cordova for his assistance with Solidworks modeling.

\bibliographystyle{ieeetr}
\bibliography{main}

\appendix
\section{Appendix}
\subsection{Detail of Regressor Matrix $\boldsymbol{Y}_{\varphi}^\nu$ }
The detail of regressor matrix $\boldsymbol{Y}_{\varphi}^\nu \in \reals^{6 \times 22}$ in (\ref{eq:regressor_def}) showed as

\begin{equation*}
\resizebox{0.98\hsize}{!}{$
    \boldsymbol{Y}_{\varphi}^\nu =\left[
    \begin{array}{c:c:c:c:c:c:c:c}
        \boldsymbol{\epsilon} & \boldsymbol{Y}_{\varphi}^\nu{(1,2)} & \mathbf{I} & \mathbf{0}_{3\times3} & \mathbf{0}_{3\times3} & \mathbf{0}_{3\times3} & \mathbf{0}_{3\times3} & \mathbf{0}_{3\times3}\\
                 &  &  &  &  &  &  & \\ \hdashline
                 &  &  &  \\
        \mathbf{0}_{3\times1} & \boldsymbol{Y}_{\varphi}^\nu{(2,2)} & \boldsymbol{Y}_{\varphi}^\nu{(2,3)} & \boldsymbol{Y}_{\varphi}^\nu{(2,4)} & \boldsymbol{\omega}_*^\times & \boldsymbol{\theta}_*^\times & \mathbf{I} & \mathbf{I} 
    \end{array}
\right]$}
\label{eq::detail-regresor}
\end{equation*}
where 
\begin{equation*}
    \begin{aligned}
    \boldsymbol{Y}_{\varphi}^\nu{(1,2)} &= -\boldsymbol{o}^{\times}\mathbf{^{\alpha}_{\mathit{o}}R} + \boldsymbol{\alpha}^{\times}\boldsymbol{\epsilon}_*^{\times}\mathbf{^{\alpha}_{\mathit{o}}R}- \boldsymbol{\omega}^{\times}\boldsymbol{o}^{\times}\mathbf{^{\alpha}_{\mathit{o}}R}+ \boldsymbol{\omega}_*^{\times}\boldsymbol{o}_*^{\times}\mathbf{^{\alpha}_{\mathit{o}}R}\\
    \boldsymbol{Y}_{\varphi}^\nu{(2,2)} &= \boldsymbol{\epsilon}^{\times}\mathbf{^{\alpha}_{\mathit{o}}R} + \boldsymbol{\omega}^{\times}\boldsymbol{o}^{\times}\mathbf{^{\alpha}_{\mathit{o}}R}-\boldsymbol{\alpha}^{\times}\boldsymbol{\epsilon}_*^{\times}\mathbf{^{\alpha}_{\mathit{o}}R}-\boldsymbol{o}^{\times}\boldsymbol{v}^{\times}\mathbf{^{\alpha}_{\mathit{o}}R}\\
    \boldsymbol{Y}_{\varphi}^\nu{(2,3)} &= \mathbf{^{\alpha}_{\mathit{o}}R}(\mathbf{^{\alpha}_{\mathit{o}}R}^{-1}\boldsymbol{o})^+ + \boldsymbol{\omega}_*^{\times}\boldsymbol{o}_*^{\times}\mathbf{^{\alpha}_{\mathit{o}}R}+ \boldsymbol{\omega}^\times(\mathbf{^{\alpha}_{\mathit{o}}R}^{-1}\boldsymbol{o})^\dagger\\
    \boldsymbol{Y}_{\varphi}^\nu{(2,4)} &= (\boldsymbol{o}^\times\mathbf{^{\alpha}_{\mathit{o}}R})^\dagger
\end{aligned}
\end{equation*}
$\boldsymbol{\epsilon}, \boldsymbol{o} \in \reals^{3}$ defined in (\ref{eq:error}) and $\boldsymbol{\theta}_* \in \reals^3$ defined in (\ref{eq::rule1_2}) and so on for its derivative form $\boldsymbol{\omega}_* \in \reals^3$.

\subsection{Proof of Matrix $\boldsymbol{M}_\nu$ Positive Definiteness}
\label{appendixproof1}
Here we demonstrate the matrix $\boldsymbol{M}_\nu$ we mentioned above to prove its properties for the proof of \textit{Theorem 1}. 
\begin{lemma}
$\boldsymbol{M}_\nu$ as given in (\ref{eq:all_wrench}) is symmetric and positive definite.
\end{lemma}

\begin{proof}
$\boldsymbol{M}_\nu$ is given as
\begin{equation*}
    \boldsymbol{M}_\nu = \boldsymbol{M}_{b} + \boldsymbol{M}_{s}
\end{equation*}
where $\boldsymbol{M}_{b}$ is a symmetric and positive definite matrix and for $\boldsymbol{M}_{s}$ using the properties in \cite{silvester2000determinants} we know for block matrix $\boldsymbol{M}_{s} \in \reals^{(a+b)\times(a+b)}$ given the form like
\begin{equation*}
\boldsymbol{M}_{s} = 
    \begin{bmatrix}
        \mathbf{A} & \mathbf{B}\\
        \mathbf{C} & \mathbf{D}
    \end{bmatrix}
\end{equation*}
where $\mathbf{A}\in \reals^{a \times a}, \mathbf{C} \in \reals^{a \times b}, \mathbf{D}\in \reals^{b \times b}, a = b = 3$
then we can write
\begin{equation*}
    \operatorname{det}(\boldsymbol{M}_{s}-\lambda \mathbf{I})=\operatorname{det}(\mathbf{A}-\lambda \mathbf{I}) \operatorname{det}(\mathbf{D}-\lambda \mathbf{I})
\end{equation*}
Thus, the eigenvalues of $\boldsymbol{M}_{s}$ are equal to the union of those of $\mathbf{A}$ and $\mathbf{D}$, which have positive real parts. Thus, $\boldsymbol{M}_{s}$ is positive definite.
Then, based on the properties in \cite{boyd2004convex} the addition of two positive definite matrices is still positive definite. Due to $\boldsymbol{B} = \boldsymbol{C}^\top$ in $\boldsymbol{M}_{s}$, so $\boldsymbol{M}_{s}$ is symmetric.

Then we can conclude that $\boldsymbol{M}_{\nu} = \boldsymbol{M}_{b} + \boldsymbol{M}_{s}$ is symmetric and positive definite. 

\end{proof}

\end{document}